\documentclass[11pt,a4paper]{article}
\usepackage{hyperref}
\usepackage{natbib}
\usepackage{times}
\usepackage[T1]{fontenc}
\usepackage[utf8]{inputenc}
\usepackage[english]{babel}

\usepackage{latexsym}
\usepackage{amsmath,amssymb,amsfonts,amsthm,amstext,amscd}
\usepackage{mathtools}
\usepackage{multirow}
\usepackage{graphicx}
\usepackage{xcolor}

\usepackage{booktabs}
\usepackage{url}
\usepackage{siunitx}
\usepackage{appendix}
\usepackage{enumerate}
\usepackage{float}

\usepackage{array,makecell}

\def\HH{\mathcal{H}}
\def\H{\operatorname{H}}
\def\Vol{\operatorname{Vol}}

\def\PP{\mathbb{P}}
\def\QQ{\mathbb{Q}}

\def\gKL{gKL}

\DeclareMathOperator*{\supp}{supp}

\DeclareMathOperator*{\argmin}{arg\,min}

\def\dd{\,\mathrm{d}}

\newcommand{\Eqref}[1]{Eq.~\ref{#1}}

\newtheorem{thm}{Theorem}[section]
\newtheorem{defi}[thm]{Definition}

\title{Sentence length}
\date{}
\author{Gábor~Borbély\hspace*{40mm}  András~Kornai\\
         {\tt \{borbely,kornai\}@math.bme.hu}\\
        Department~of~Algebra, Budapest~University of Technology and Economics
        }

\begin{document}
\maketitle
\begin{abstract}
The distribution of sentence length in ordinary language is not well captured
by the existing models. Here we survey previous models of sentence length and
present our random walk model that offers both a better fit with the data and
a better understanding of the distribution. We develop a generalization of KL
divergence, discuss measuring the noise inherent in a corpus, and present a
hyperparameter-free Bayesian model comparison method that has strong
conceptual ties to Minimal Description Length modeling. The models we obtain 
require only a few dozen bits, orders of magnitude less than the naive
nonparametric MDL models would. 
\end{abstract}

\section{Introduction}

Traditionally, statistical properties of sentence length distribution were
investigated with the goal of settling disputed authorship
\citep{Mendenhall:1887,Yule:1939}. Simple models, such as a ``monkeys and
typewriters'' Bernoulli process \citep{Miller:1957} do not fit the data well,
and this problem is inherited from n-gram Markov to n-gram Hidden Markov
models, such as found in standard language modeling tools like SRILM
\citep{Stolcke:2011}. Today, length modeling is used more often as a
downstream task to probe the properties of sentence vectors
\citep{Adi:2017,Conneau:2018}, but the problem is highly relevant in other
settings as well, in particular for the current generation of LSTM/GRU-based
language models that generally use an ad hoc cutoff mechanism to regulate
sentence length.  The first modern study, interested in the entire shape of
the sentence-length distribution, is \citet{Sichel:1974}, who briefly
summarizes the earlier proposals, in particular negative binomial
\citep{Yule:1944}, and lognormal \citep{Williams:1944}, being rather critical
of the latter:

\begin{quote}
The lognormal model suggested by Williams and used by Wake must be rejected on
several grounds: In the first place the number of words in a sentence
constitutes a discrete variable whereas the lognormal distribution is
continuous. \citet{Wake:1957} has pointed out that most observed
log-sentence-length distributions display upper tails which tend towards zero
much faster than the corresponding normal distribution.  This is also evident
in most of the cumulative percentage frequency distributions of
sentence-lengths plotted on log-probability paper by Williams (1970). The
sweep of the curves drawn through the plotted observations is concave upwards
which means that we deal with sub-lognormal populations. In other words, most
of the observed sentence-length distributions, after logarithmic
transformation, are negatively skew.  Finally, a mathematical distribution
model which cannot fit real data --as shown up by the conventional $\chi^2$
test-- cannot claim serious attention. \cite[][p. 26]{Sichel:1974}
\end{quote}

Sichel's own model is a mixture of Poisson distributions given as
\begin{equation}\label{mix-poisson}
\phi(r)=\frac{\sqrt{1-\theta}^{\gamma}}{K_{\gamma}(\alpha\sqrt{1-\theta})}
\frac{(\alpha\theta/2)^r}{r!} K_{r+\gamma}(\alpha) 
\end{equation}
where $K_{\gamma}$ is the modified Bessel function of the second kind of order
$\gamma$. As Sichel notes, ``a number of known discrete distribution functions
such as the Poisson, negative binomial, geometric, Fisher's logarithmic series
in its original and modified forms, Yule, Good, Waring and Riemann
distributions are special or limiting forms of (1)''. While Sichel's own
proposal certainly cannot be faulted on the grounds enumerated above, it still
leaves something to be desired, in that the parameters $\alpha, \gamma,
\theta$ are not at all transparent, and the model lacks a clear genesis. In
Section~\ref{sec:model} of this article we present our own model aimed at
remedying these defects and in Section~\ref{sec:three} we analyze its
properties. Our results are presented is Section~\ref{sec:four}. The relation
between the sentence length model and grammatical theory is discussed in the
concluding Section~\ref{sec:five}.

\section{The random walk model}
\label{sec:model}
In the following Section we introduce our model of random walk(s).  The
predicted sentence length is basically the return time of these stochastic
processes, i.e. the probability of a given length is the probability of the
appropriate return time.  

Let $X_k$ be a random walk on $\mathbb{Z}$ and $X_k(t)$ the position of the
walk at time $t$. Let $X_k(0)=k$ be the initial condition. The walk is given
by the following parameters:
\begin{align}\label{walk}
    X_k(t+1) - X_k(t) &=
        \begin{cases}
        -1 & \text{with probability }p_{-1}\\
        0 & \text{with probability }p_{0}\\
        1 & \text{with probability }p_{1}\\
        2 & \text{with probability }p_{2}
        \end{cases}
\end{align}
The random walk is the sum of these independent steps. (2) is a simple model
of valency (dependency) tracking: at any given point we may introduce, with
probability $p_2$, some word with two open valences (e.g. a transitive verb),
with probability $p_1$ one that brings one new valence (e.g. an intransitive
verb or an adjective), with probability $p_0$ one that doesn't alter the count
of open valencies (e.g. an adverbial), and with probability $p_{-1}$ one that
fills an open valency, e.g. a proper noun. More subtle cases, where a single
word can fill more than one valency, as in Latin {\it accusativus cum
  infinitivo} or (arguably) English equi, are discussed in
Section~\ref{sec:five}. The return time is defined as

\begin{equation}
\tau_k = \min_{t\geq 0}\{t : X_k(t)=0\}
\end{equation}
\begin{figure}
\includegraphics[width=0.45\textwidth]{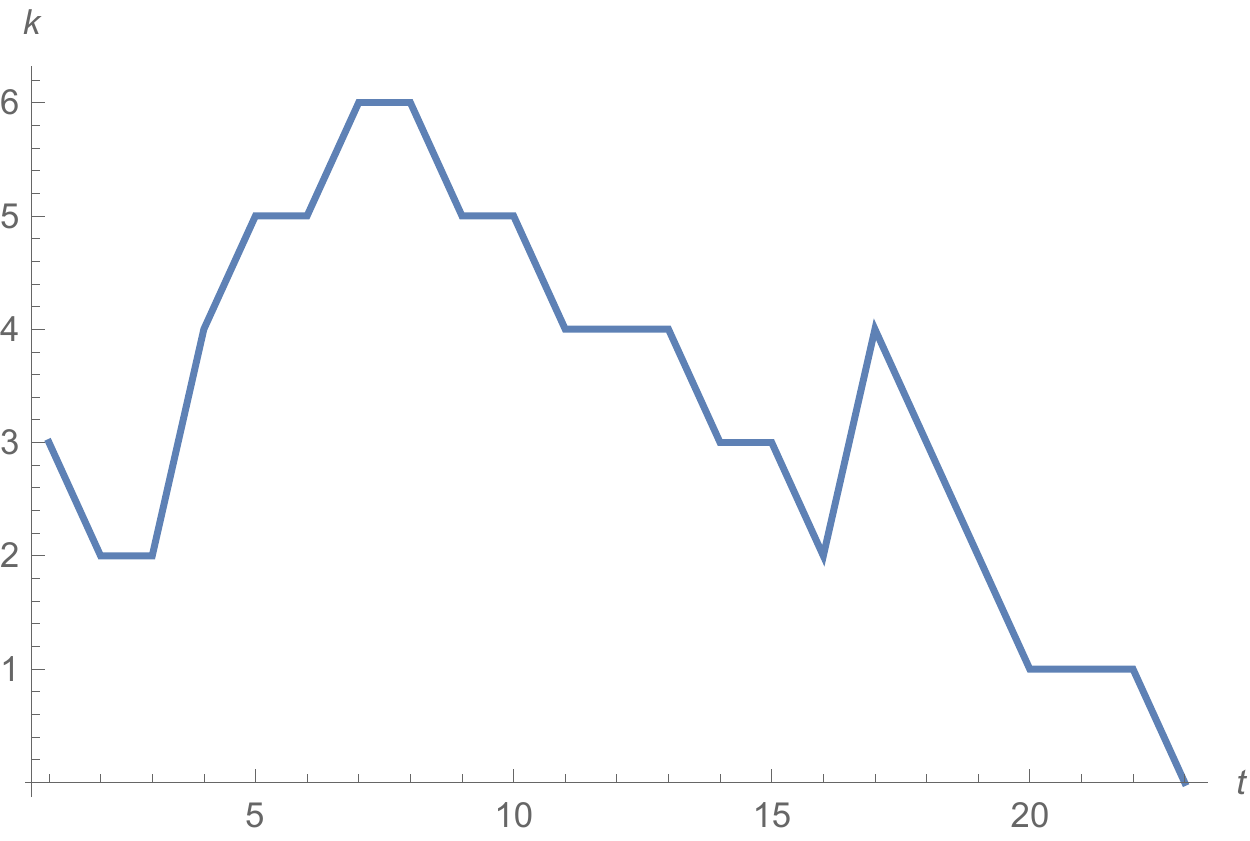}
\caption{Sentence length is modeled as the return time of a random walk.}
\end{figure}
In particular, $\tau_1$ is the time needed to go from $1 \to 0$.
We will calculate the probability-generating function to find the probabilities.
\begin{equation}
f(x) \coloneqq \mathbb{E} \left(x^{\tau_1}\right)
\end{equation}
The generating function of $\tau_k$ easily follows from $\tau_1$, since
$\tau_k$ is the sum of $k$ independent copies of $\tau_1$, so 
the generating function of $\tau_k$ is simply $f(x)^k$.

In order to calculate $f(x)$, we condition on the first step:
\begin{align}
f(x) = \ & p_{-1}\cdot x + & \text{finishing in one step} \nonumber\\
& p_0\cdot x \cdot f(x) + & \text{wait $\tau_1$ again} \nonumber\\
& p_1\cdot x \cdot f(x)^2 + & \text{wait $\tau_1$ two times}\nonumber\\
& p_2\cdot x \cdot f(x)^3 & \text{wait $\tau_1$ three times}
\label{condition_on_first_step}
\end{align}
Therefore, $f(x)$ is the solution of the following equation (solved for $f$ and $x$ is a parameter):
\begin{equation}
\label{momentum_generator}
p_{-1}\cdot x + (p_0\cdot x-1) \cdot f + p_1\cdot x \cdot f^2 + p_2\cdot x \cdot f^3 = 0
\end{equation}

This can be solved with Cardano's formula. The probabilities are given by
\begin{equation}
\PP(\tau_k=i) = [x^i] {f(x)}^k = \frac{1}{i!}{\left.\frac{\partial^i}{\partial x^i}{f(x)}^k\right|}_{x=0}
\label{probabilities}
\end{equation}

For given parameters $p_{-1}, p_0, p_1,p_2$ and $k$, and a given $i$,
one can evaluate these probabilities numerically,
but we need a bit more analytical form.
Let us define the followings.
\begin{align}
F(u) & = p_{-1} + p_0\cdot u+p_1\cdot u^2 +p_2 u^3\\
g(f) & = \frac{f}{F(f)}
\label{step_generator}
\end{align}
With these functions \Eqref{momentum_generator} becomes $x=g(f(x))$,
meaning that we are looking for the inverse function of $g$. 
One can see that $g(0)=0$ and $g'(0)=1/p_{-1}\neq 0$,
therefore we can apply the Lagrange inversion theorem.
Calculations detailed in the Appendix yield the following formula.
\begin{equation}
\PP(\tau_k=i) = \frac{k}{i}[u^{i-k}]F^i(u)
\label{lagrange_inversion}
\end{equation}

\noindent
Since $F$ is a polynomial one can calculate its powers by polynomial
multiplication and get $\PP(\tau_k=i)$ by looking up the appropriate
coefficient.  Here $k$ is an integer (discrete) model parameter and $p_{-1},
p_0, p_1, p_2$ are real (continuous) numbers.  That makes the above mentioned
probabilities \emph{differentiable} in the continuous parameters.

We call the parameter $k$, the starting point of the random walk, the total
valency.  Note that $\tau_k \geq k$ with probability $1$, therefore one cannot
model the sentences shorter then $k$.  To overcome this obstacle, we introduce
the mixture model that consists of several models with various $k$ values and
coefficients for convex linear combination.

\begin{equation}
\PP_{k_1, \alpha_1, k_2, \alpha_2, \ldots k_m, \alpha_m} (\tau = i) = 
\sum_{j=1}^m \alpha_j\cdot \PP(\tau_{k_j}=i)
\label{mixture_model}
\end{equation}
where the parameters $\alpha_j$ are mixture coefficients; positive and sum up to $1$.
Also every term in the mixture have different $p_{-1}, p_0, p_1$ and $p_2$ values (all positive and sum up to one).
In this way, we can model the sentences with length at least $\min_j k_j$.

\begin{figure}
\def\arraystretch{1.2}
\begin{tabular}{c|c|cccc|}
\cline{2-6}
 $k_1$ & $\alpha_1$ & $p^1_{-1}$ & $p^1_{0}$ & $p^1_{1}$ & $p^1_{2}$ \\\cline{3-6}
 $k_2$ & $\alpha_2$ & $p^2_{-1}$ & $p^2_{0}$ & $p^2_{1}$ & $p^2_{2}$ \\\cline{3-6}
\vdots & \vdots & \multicolumn{4}{c|}{\vdots} \\\cline{3-6}
 $k_m$ & $\alpha_m$ & $p^m_{-1}$ & $p^m_{0}$ & $p^m_{1}$ & $p^m_{2}$\\
 \cline{2-6}
\end{tabular}
\caption{Model parameters. The framed parameters are real, positive numbers and should sum up to $1$.}
\label{parameters}
\end{figure}

Theoretically, there is no obstacle to have different number of $p$ values to different $k$ values.
The model can be a mixture of random walks, where the individual processes can have different upward steps.

\section{Model analysis}
\label{sec:three}
Here we introduce and analyze the experimental setup that we will use in
Section~\ref{sec:four} to fit our model to various datasets.  The raw data is
a set of positive integers, the sentence lengths, and their corresponding
weights (absolute frequencies) $\{n_x\}_{x\in X}$. We call
$n\coloneqq\sum_{x\in X}n_x$ the {\it size} and $X$ the {\it support} of the
data. Since the model is differentiable in the continuous parameters
(including the mixing coefficients), the direct approach would be to perform
gradient descent on the dissimilarity as an objective function to find the
parameters.  With fixed valency parameters $k_j$ this is a constrained
optimization task $\operatorname{dist}(\PP_\text{sample},
\PP_\text{modeled})\to\min$.

In some cases, especially for smaller datasets, we might find it expedient to
bin the data, for example \cite{Adi:2017} use bins (5-8), (9-12), (13-16),
(17-20), (21-25), (26-29), (30-33), and (34-70). On empirical data (for
English we will use the BNC\footnote{\url{http://www.natcorp.ox.ac.uk}} and
the UMBC
Webbase\footnote{\url{https://ebiquity.umbc.edu/resource/html/id/351}} and for
other languages the SZTAKI
corpus\footnote{\url{http://hlt.sztaki.hu/resources/webcorpora.html}}) this
particular binning leaves a lot to be desired. We discuss this matter in
Section~\ref{3.1}, together with the choice of dissimilarity (figure of
merit).  An important consideration is that a high number of mixture
components fit the data better but have more model parameters -- this is
discussed in \ref{3.2}.

\subsection{Short utterances}\label{3.1}

Short utterances such as imperatives {\it Stop!} or {\it Help} are common both
in spoken corpora and in written materials, both in fiction, where incomplete
sentences abound, especially in dialog intended to sound natural, and in
nonfiction, where they are encountered often in titles, subtitles, and
itemized lists. The prevailing tokenization convention, where punctuation is
counted as equivalent to a full word, also has an effect on the distribution,
more perceptible at the low end.

Since the eight bins used by \cite{Adi:2017} actually ignore the very low
(1-4) and very high (71+) ranges of the data, we will use ordinary deciles,
setting the ten bins as the data dictates. In this regard, it is worth noting
that in the 18 non-English corpora used in this study the neglected low bin
contains on the average 17.4\% of the data (variance 6.3\%, low 8.1\% on
Romanian, high 33.7\% on Serbian\_sr). Besides tokenization, perhaps the most
important factor is morphological complexity, since in highly agglutinating
languages a single word is sufficient for what would require a multiword
sentence in English, as in Hungarian {\it elvihetlek} `I can give you a ride'.

At the high end (sentences with 71 or more words) the original binning omits
on the average 4.3\% of the data (variance 3.1\%, low 1.2\% Nynorsk, high
15.1\% Serbian\_sr). Comparable figures for English are 3.7\% (UMBC) and 14.4\%
(BNC) for the low bin, 1.0\% (UMBC) and 0.8\% (BNC) for the high bin. The last
column of Table \ref{datasets} shows the length of the longest sentence in each of the
subcorpora considered. Since the number of datapoints is high, ranging from
1.3m (Nynorsk) to 136.6m (UMBC), the conventional $\chi^2$ test does not
provide a good figure of merit on the original data (no fit is ever
significant, especially as there is a lot of variation at the high end where 
only few lengths are extant), nor on the binned data, where every fit is 
highly significant. 

\begin{table}[H]{\centering
\hspace*{-15mm}\begin{tabular}{cl|r|c|crr}
 & \multirow{2}{*}{dataset} & \multirow{2}{*}{\makecell{number of\\sentences}} & \multirow{2}{*}{\makecell{tolerance\\(in nats)}} & \hspace*{-11mm}mean\hspace*{-15mm} & \hspace*{-10mm}$99.9\%$\hspace*{-6mm} & max \\
 &                          &                                      &                                      & \multicolumn{3}{c}{sentence length} \\ \toprule
\multirow{10}{*}{\rotatebox[origin=c]{90}{BNC-2.0 (English)}} & 
  BNC-A &  753442 & 5.669e-4 & 20.967 & 97 & 555 \\ \cline{2-7}
& BNC-B &  362003 & 9.051e-4 & 20.650 & 96 & 365 \\ \cline{2-7}
& BNC-C &  955486 & 3.636e-4 & 20.524 & 102 & 491 \\ \cline{2-7}
& BNC-D &    6138 & 4.033e-2 & 16.366 & 228 & 466 \\ \cline{2-7}
& BNC-E &  337370 & 7.640e-4 & 22.219 & 106 & 763 \\ \cline{2-7}
& BNC-F &  527758 & 1.031e-3 & 19.351 & 130 & 2208 \\ \cline{2-7}
& BNC-G &  478860 & 7.835e-4 & 18.753 & 106 & 435 \\ \cline{2-7}
& BNC-H & 1185549 & 4.482e-4 & 18.841 & 118 & 950 \\ \cline{2-7}
& BNC-J &  359352 & 1.446e-3 & 18.666 & 156 & 1100 \\ \cline{2-7}
& BNC-K & 1086242 & 5.039e-4 & 12.784 & 116 & 918 \\ \midrule
& UMBC & 136630947 & 2.442e-3 & 24.434 & 116 & 3052 \\ \midrule
\multirow{18}{*}{\rotatebox[origin=c]{90}{SZTAKI corpus}}
& Catalan    &  23927377 & 1.751e-3 & 27.496 & 384 & 5279 \\ \cline{2-7}
& Croatian   &  62196524 & 5.616e-3 & 23.975 & 369 & 8598 \\ \cline{2-7}
& Czech      &  30382696 & 5.147e-3 & 20.139 & 285 & 6081 \\ \cline{2-7}
& Danish     &  26687240 & 7.557e-3 & 18.593 & 296 & 16425 \\ \cline{2-7}
& Dutch      & 103958658 & 2.408e-3 & 19.135 & 296 & 16128 \\ \cline{2-7}
& Finnish    &  58104101 & 1.946e-3 & 15.538 & 237 & 5552 \\ \cline{2-7}
& Indonesian &  13095607 & 1.231e-2 & 23.675 & 343 & 22762 \\ \cline{2-7}
& Lithuanian &  81826291 & 1.184e-3 & 17.170 & 294 & 21857 \\ \cline{2-7}
& Bokm\aa l  &  84375397 & 3.564e-3 & 19.199 & 281 & 14032 \\ \cline{2-7}
& Nynorsk    & 1393312 & 3.946e-3 & 18.836 & 175 & 1591 \\ \cline{2-7}
& Polish     &  72983880 & 8.508e-3 & 19.549 & 396 & 24353 \\ \cline{2-7}
& Portuguese &  37953728 & 4.973e-2 & 25.365 & 448 & 9614 \\ \cline{2-7}
& Romanian   &  36211510 & 2.338e-2 & 29.466 & 473 & 54434 \\ \cline{2-7}
& Serbian.sh &  35606837 & 4.531e-3 & 23.744 & 332 & 6800 \\ \cline{2-7}
& Serbian.sr &   2023815 & 7.189e-3 & 37.736 & 862 & 6800 \\ \cline{2-7}
& Slovak     &  39633566 & 2.572e-3 & 21.759 & 402 & 24571 \\ \cline{2-7}
& Spanish    &  47673229 & 8.365e-4 & 29.305 & 471 & 29183 \\ \cline{2-7}
& Swedish    &  54218846 & 2.526e-3 & 16.468 & 315 & 8127
\end{tabular}
\caption{Sentence length datasets. For tolerance see \ref{3.2}}
\label{datasets}
}\end{table}

A better choice is the Kullback--Leibler divergence, but this still suffers
from problems when the supports of the distributions do not coincide. In our
case we have this problem both at the low end, where the model predicts
$\PP(\tau = i) = 0$ for $i<k$, and at the high end, where we predict positive
(albeit astronomically small) probabilities of arbitrarily long sentences.  To
remedy this defect, we define generalized KL divergence, $gKL$, as follows.

\begin{defi}[Motivated by Theorem \ref{main_theorem}.]
\label{generalized_kl_def}
Let $\PP$ and $\QQ$ be probability measures over the same measurable space $(X, \Sigma)$ that are both
absolutely continuous with respect to a third measure $\dd x$, and let
$\lambda$ be $\PP(\supp(\PP)\cap\supp(\QQ))$. Then
\begin{equation}
\begin{split}
\gKL(\PP, \QQ) & \coloneqq 
-\lambda\cdot \ln\lambda + \\
& \ \int\limits_{\supp(\PP) \cap \supp(\QQ)} \PP(x) \cdot\ln \frac{\PP(x)}{\QQ(x)} \dd x
\end{split}
\end{equation}
\end{defi}

Clearly, $\gKL$ reduces to the usual KL divergence if the support of the
distributions coincide. Perhaps the high end of the distribution could be
ignored, at least for English, at the price of losing 1\% of the data, but
ignoring the short sentences, 14.4\% of the BNC, is hard to countenance. As a
practical matter this means we need to bring in mixture components with total
valency $k < 4$, and these each bring 4 parameters (the mixture weight and 3
$p_i$ values) in tow. Obviously, the more components we use, the better the
fit will be, so we need to control the trade-off between these.  In
Section~\ref{3.2} we introduce a method derived from Bayesian model comparison
\cite{MacKay:2003} that will remedy the zero modeled probabilities and answer
the model complexity trade-off.

\subsection{Bayesian model comparison}\label{3.2}

If a dataset $D$ has support $X$, with $n_x> 0$ being the number that length
$x$ occurred, the data size is $|D| = \sum_{x\in X} n_x$ and the observed
probabilities are $p_x \coloneqq \frac{n_x}{|D|}$. Let $\HH_i$ be
$i^\text{th}$ model in some list of models. Each model is represented by a
parameter vector $\mathbf{w}$ in a (continuous) parameter space, and
$\supp\HH_i = \{x\mid \PP(x\mid \HH_i)>0 \}$ is not necessarily equal to
$X$. Clearly, different $\HH_i$ may have different support, but a given model
has the same support for every $\mathbf{w}$. Model predictions are given by 
$\QQ_{\mathbf{w}_i}(x) \coloneqq \PP(x\mid \mathbf{w}_i)$, and the 
\textbf{evidence} the $i^\text{th}$ model has is 

\begin{equation}
\PP(\HH_i\mid D) = \frac{\PP(D\mid\HH_i)\cdot \PP(\HH_i)}{\PP(D)}
\end{equation}

\noindent
If one supposes that no model is preferred over any other models ($\PP(\HH_i)$ is constant)
then the decision simplifies to finding the model that maximizes
\begin{equation}
\PP(D\mid \HH_i) = \int_{\HH_i} \PP(D\mid \mathbf{w}_i, \HH_i)\cdot \PP(\mathbf{w}_i \mid \HH_i)\dd \mathbf{w}_i
\end{equation}

\noindent
We make sure that no model parameter is preferred by setting a uniform prior:
\begin{equation}
\PP(\mathbf{w}_i \mid \HH_i) = 1/\left(\int_{\HH_i} 1\dd \mathbf{w}_i\right)  =1/{\Vol(\HH_i)}
\end{equation}

\noindent
We estimated this integral with Laplace's method by introducing
$f(\mathbf{w}_i) \coloneqq -\frac{1}{|D|}\ln \QQ_{\mathbf{w}_i}(D)$, i.e. the
cross entropy (measured in nats).
\begin{align}
\PP(D\mid \mathbf{w}_i, \HH_i) & = \prod_{x\in X}\QQ_{\mathbf{w}_i}(x)^{n_x} \nonumber
\\
f(\mathbf{w}_i) & = -\sum_{x\in X} p_x \cdot \ln \QQ_{\mathbf{w}_i}(x)
\end{align}

\noindent
Taking $-\frac{1}{|D|}\ln(\bullet)$ of the evidence amounts to minimizing in
$i$ the following quantity:
\begin{gather}
\label{non_fisher_overall}
    f(\mathbf{w}^\ast_i) + 
    \frac{1}{|D|}\cdot\ln\Vol(\HH_i) + \\
    \nonumber
    \frac{1}{2|D|} \ln\det f''(\mathbf{w}^\ast_i) + 
    \frac{d}{2|D|} \cdot \ln\frac{|D|}{2\pi}
\end{gather}
where $d$ is the dimension of $\HH_i$ (number of parameters), $f''$ is the
Hessian and $\mathbf{w}^\ast_i = \argmin_{\mathbf{w}_i\in\HH_i}
f(\mathbf{w}_i)$ for a given $i$.  Since the theoretical optimum of
$f(\mathbf{w}_i)$ is the entropy of the data ($\ln 2\cdot\H(D)$), we subtract
this quantity from \Eqref{non_fisher_overall} so that the term
$f(\mathbf{w}_i)$ becomes the relative entropy (measured in nats) with a
theoretical minimum of 0. 

We introduce an augmented model to deal with the datapoints where $\QQ_{\mathbf{w}_i}(x) = 0$.
\begin{equation}
\label{augmented_model}
\overline{\QQ}_{\mathbf{w}_i,\mathbf{q}}(x) \coloneqq \begin{cases}
\lambda\QQ_{\mathbf{w}_i}(x) & \text{if } \QQ_{\mathbf{w}_i}(x) > 0 \\
(1-\lambda) q_x & \text{if $n_x>0, \QQ_{\mathbf{w}_i}(x)=0$}
\end{cases}
\end{equation}
where
\begin{align*}
\lambda & = \sum_{x\in X} p_x & \text{covered probability}\\
1-\lambda & = \sum_{x\in X\setminus \supp(\HH_i)} p_x & \text{uncovered probability}
\end{align*}

The newly introduced model parameters $\mathbf{q} = (q_x)_{x\in X\setminus
  \supp(\HH_i)}$ are also constrained: they have to be positive and sum up to
one, i.e. inside the probability simplex.  After finding the optimum of
$\mathbf{q}$ and modifying \Eqref{non_fisher_overall} with the auxiliary terms
and subtracting the entropy of the data ($\ln 2\cdot \H(D)$) as discussed
above, one gets:

\begin{align}
\nonumber
& -\lambda\cdot \ln \lambda + \sum_{x\in X\cap\supp(\HH_i)}
        p_x\cdot \ln\frac{p_x}{\QQ_{\mathbf{w}^\ast_i}(x)} + \\
\nonumber
& \frac{1}{|D|} \cdot \left( \ln\Vol(\HH_i)+\ln\Vol(\text{aux. model})\right) + \\
\nonumber
& \frac{1}{2|D|} \cdot \ln\det\left(\text{model Hessian}\right) + \\
\nonumber
& \frac{1}{2|D|} \cdot \ln\det\left(\text{aux. model Hessian}\right) + \\
& \frac{d'}{2|D|} \cdot \ln\frac{|D|}{2\pi}
\label{augmented_evidence}
\end{align}
where $d'$ is the original model dimension plus the auxiliary model
dimension. For sufficiently large corpora ($|D| \rightarrow \infty$) all but
the first term will be negligible, meaning that the most precise model (in
terms of $\gKL$ divergence) wins regardless of model size. One way out would
be to choose an `optimum corpus size' \citep{Zipf:1949}, a move that has
already drawn strong criticism in \citet{Powers:1998} and one that would
amount to little more than the addition of an extra hyperparameter to be set
heuristically.

Another, more data-driven approach is based on the observation that corpora
have {\it inherent noise}, measurable as the KL divergence between a random
subcorpus and its complement \citep{Kornai:2013b} both about the same size
(half the original). Here we need to take into account the fact that large 
sentence lengths appear with frequency 1 or 0, so subcorpora $D_1$ and
$D_2=D\setminus D_1$ will not have the exact same support as the original, and
we need to use symmetrized gKL: the \textbf{inherent noise} $\delta_D$ of a corpus
$D$ is $\frac{1}{2}(\gKL(D_1,D_2) + \gKL(D_2, D_1))$, where $D_1$ and $D_2$ are
equal size subsets of the original corpus $D$, and the \gKL\ divergence is
measured on their empirical distributions.  

$\delta_D$ is largely independent of the choice of subsets $D_1, D_2$ of the
original corpus, and can be easily estimated by randomly sampled $D_i$s. To
the extent crawl data and classical corpora are sequentially structured
\citep{Curran:2002}, we sometimes obtain different noise estimates based on
random $D_i$ than from comparing the first to the second half of a corpus, the
procedure we followed here.  In the Minimum Description Length (MDL) setting
where this notion was originally developed it is obvious that we need not
approximate corpora to a precision better than $\delta$, but in the Bayesian
setup we use here matters are a bit more complicated.

\begin{defi}
For $\delta > 0$ let
\begin{equation}\label{tolerable}
\gKL_\delta(\PP, \QQ) \coloneqq  \max(0, \gKL(\PP, \QQ)-\delta)
\end{equation}
For a sample $\PP$ with inherent noise $\delta$,
a model $\QQ$ is called \textbf{tolerable} if $\gKL_\delta(\PP, \QQ) = 0$
\end{defi}

If $\gKL_\delta$ is used instead of $\gKL$ in \Eqref{augmented_evidence}
then model size $d$ becomes important.  If a model
fits within $\delta$ then the first term becomes zero and for large $|D|$
values the number of model parameters (including auxiliary parameters) will
dominate the evidence. The limiting behavior of our evidence formula, with
tolerance for inherent noise, is determined by the following observations:

\begin{enumerate}
\item Any tolerable model beats any non-tolerable one.
\item If two models are both tolerable and have different number of model parameters (including auxiliary model),
then the one with the fewer parameters wins.
\item If two models are both tolerable and have the same number of parameters,
then the model volume and Hessian decides.
\end{enumerate}

An interesting case is when no model can reach the inherent noise -- in this
case we recover the original situation where the best fit wins, no matter the
model size.

\section{Results}\label{sec:four}

A single model $\HH_i$ fit to some dataset is identified by its {\it order},
defined as the number of upward steps the random walk can take at once: $1, 2$
or $3$, marked by the number before the first decimal; and its {\it mixture}, a
non-empty subset of $\{1,2,3,4,5\}$ that can appear as $k$: valency of a
single component (for example {\bf k1.2.4} marks $\{1,2,4\}\subseteq
\{1,2,3,4,5\}$).  Altogether, we trained $3\times31 = 93$ locally optimal
models for each dataset and compared them with \Eqref{augmented_evidence},
except that $gKL_\delta$ is used with the appropriate tolerance.

We computed $\mathbf{w}^\ast$ with a (non-batched) gradient descent
algorithm.\footnote{You can find all of our code used for training and
  evaluating at \url{https://github.com/hlt-bme-hu/SentenceLength}} We used
Adagrad with initial learning rate $\eta=0.9$, starting from uniform $p$ and
$\alpha$ values, and iterated until every coordinate of the gradient fell
within $\pm10^{-3}$. The gradient descent typically took $10^2-10^3$
iterations to reach a plateau, but about .1\% of the models were more sensitive and
required a smaller learning rate $\eta=0.1$ with more (10k) iterations.

\subsection{Validation}

The model comparison methodology was first tested on artificially generated
data.  We generated 1M+1M samples of pseudo-random walks with parameters:
$p_{-1} = 0.5, p_0 = p_1 = 0.25$ (at most one step upward) and $k = 3$ (no
mixture) and obtained the inherent noise and length distribution.  The
inherent noise was about 3.442e-4 nats. We trained all 93 models and compared
them as described above.  

The validation data size is $2\cdot 10^6$ but we also replaced $|D|$ with a
hyper-parameter $n$ in \Eqref{augmented_evidence}.  This means that we
faked the sample to be bigger (or smaller) with the same empirical
distribution.  We did this with the goal of imitating the `optimum corpus
size' as an adverse effect.

As seen on Table \ref{test_results} the true model wins.  We also tested the
case when the true model was simply excluded from the competing models. In
this case, the tolerance is needed to ensure a stable result as $n\to\infty$.

\begin{table}[H]{\centering
\setlength\tabcolsep{1pt}
\begin{tabular}{l*{6}{|l}}
 \multirow{2}{*}{1.k3 artificial data} & \multicolumn{6}{c}{best parameters for various $n$ values} \\
 & \multicolumn{1}{c|}{1k} & \multicolumn{1}{c|}{10k} & \multicolumn{1}{c|}{100k} & \multicolumn{1}{c|}{1M} & \multicolumn{1}{c|}{10M} & \multicolumn{1}{c}{1G} \\\toprule
with tolerance & 3.k1-5 &1.k3  & 1.k3 & 1.k3 & 1.k3 & 1.k3 \\\hline
w/o tolerance & 3.k1-5 & 1.k3 & 1.k3 & 1.k3 & 1.k3 & 1.k3 \\\midrule
w tolerance, -true & 3.k1-5 & 2.k4 & 2.k4 & 2.k4 & 2.k4 & 2.k4\\\hline
w/o tolerance, -true & 3.k1-5 & 2.k4 & 1.k2.3 & 1.k2.3 & 1.k2.3 & 1.k3-5
\end{tabular}
\caption{Optimal models for artificially generated data (1.k3) for various $n$ values.}
\label{test_results}
}\end{table}

As there are strong conceptual similarities between MDL methods and the
Bayesian approach \citep{MacKay:2003}, we also compared the models with MDL,
using the same locally optimal parameters as before, but encoding them in
bits.  To this end we used a technique from \cite{Kornai:2013b} where all of
the continuous model parameters are discretized on a log scale unless the
discretization error exceeds the tolerance.  The model with the least number
of bits required wins if it fits within tolerance.  (The constraints are
hard-coded in this model, meaning that we re-normalized the parameters after
the discretization.) In the artificial test example, the model 1.k3 wins,
which is also the winner of the Bayesian comparison. If the true model is
excluded, the winner is 1.k2.3.

\subsection{Empirical data}

Let us now turn to the natural language corpora summarized in
Table~\ref{datasets}. Not only are the webcrawl datasets larger than the BNC
sections, but they are somewhat noisier and have suspiciously long sentences.
To ease the computation, we excluded sentences longer than $1,000$
tokens. This cutoff is always well above the $99.9^\text{th}$ percentile
given in the next to last column of Table~\ref{datasets}. The results,
summarized in Table~\ref{optimal_w_tol}, show several major tendencies.

First, most of the models (115 out of 145,\phantom{i} 79.7\%) fit sentence length of the
entire subcorpus better than the empirical distribution of the first half
would fit the distribution of the second half. When this criterion is {\it not} met
for the best model, i.e. the \gKL \  distance of the model from the data is above
the internal noise, the ill-fitting model form is shown in {\it italics}.

Second, this phenomenon of not achieving tolerable fit is seen primarily (20
out of 30) in the first column of Table~\ref{optimal_w_tol}, corresponding to
a radically undersampled condition $n=1,000$, and (9 out of 30) to a somewhat
undersampled condition $n=10,000$. Only on the largest corpus (UMBC, 136.6m
sentences) do we see the phenomenon persist to $n=100,000$, but even there, by
setting $n=10^6$, still less than 1\% of the actual dataset size, we get a
good fit, within two thirds of the inherent noise.

Third, and perhaps most important, for sufficiently large $n$ the Bayesian
model comparison technique we advocate here actually selects rather simple
models, with order 1 (no ditransitives, a matter we return to in
Section~\ref{sec:five}) and only one or two mixture components. We emphasize
that `sufficiently large' is still in the realistic range, one does not have
to take the limit $n \rightarrow \infty$ to obtain the correct model. The last
two columns (gigadata and infinity) always coincide, and in 21 of the 29
corpora the 1M column already yield the same result.

Given that tolerance is generally small, less than 0.8 bits even in our
noisiest corpus (Portuguese), we didn't expect much change if we perform the
model comparison without using Eq.~\ref{tolerable}. Unsurprisingly, if we
reward every tiny improvement in divergence, we get more models (123 out of
145, 84.8\%) within the tolerable range -- those outside the tolerance limit
are again given in italics in Table~\ref{optimal_wo_tol}.  But we pay a heavy
price in model complexity: the best models (in the last two columns) are now
often second order, and we have to countenance a hyperparameter $n$ which
matters (e.g. for Swedish).

\begin{table}[H]{\centering
\setlength\tabcolsep{1pt}
\hspace*{-3mm}\begin{tabular}{l*{6}{|l}}
 \multirow{2}{*}{dataset} & \multicolumn{6}{c}{best parameters for various $n$ values} \\
 & \multicolumn{1}{c|}{1k} & \multicolumn{1}{c|}{10k} & \multicolumn{1}{c|}{100k} & \multicolumn{1}{c|}{1M} & \multicolumn{1}{c}{1G} & \multicolumn{1}{c}{$\infty$} \\\toprule
BNC-A & {\it 3.k1-5} & {\it 3.k2-5} &  1.k4.5 & 1.k4.5 & 1.k4.5 & 1.k4.5 \\\hline
BNC-B & {\it 3.k1-5} & 3.k1.3.5 & 3.k1.3.5 & 1.k1.5 & 1.k1.5 & 1.k1.5 \\\hline
BNC-C & 3.k2-5 & 3.k2-5 & 3.k2-5 & 3.k2-5 & 1.k1.3 & 1.k1.3 \\\hline
BNC-D & 3.k2.3.5 & 3.k2.3.5 & 3.k2.3.5 & 3.k2.3.5 & 1.k2 & 1.k2 \\\hline
BNC-E & {\it 3.k1.3-5} & 3.k2-5 & 3.k2-5 & 3.k2-5 & 1.k2.5 & 1.k2.5 \\\hline
BNC-F & {\it 3.k2-5} & 3.k1-5 & 3.k3-5 & 3.k3-5 & 1.k3 & 1.k3 \\\hline
BNC-G & 3.k1-5 & 3.k1-5 & 3.k1-5 & 1.k1.2 & 1.k1.2 & 1.k1.2 \\\hline
BNC-H & {\it 3.k2-5} & 3.k4.5 & 3.k4.5 & 3.k4.5 & 1.k4 & 1.k4 \\\hline
BNC-J & 3.k1.3-5 & 3.k1.3-5 & 3.k1.3-5 & 3.k1.3-5 & 1.k2 & 1.k2 \\\hline
BNC-K & 3.k1-5 & 3.k1-5 & 3.k2.4.5 & 3.k2.4.5 & 1.k2 & 1.k2 \\\midrule
UMBC & {\it 3.k1-5} & {\it 3.k1-5} & {\it 1.k1.4} & 1.k2.5 & 1.k2.5 & 1.k2.5 \\\midrule
Catalan & {\it 3.k2-5} & {\it 3.k2-5} & 1.k2.5 & 1.k2.5 & 1.k2.5 & 1.k2.5 \\\hline
Croatian & {\it 3.k1.3-5} & 3.k2-5 & 3.k3-5 & 1.k1.3 & 1.k1.3 & 1.k1.3\\\hline
Czech & {\it 3.k2.4.5} & {\it 3.k2.4.5} & 3.k1.2.5 & 1.k1.3 & 1.k1.3 & 1.k1.3 \\\hline
Danish & {\it 3.k1-5} & 3.k2.4.5 & 3.k2.4.5 & 1.k1.4 & 1.k1.4 & 1.k1.4 \\\hline
Dutch & {\it 3.k1-5} & {\it 3.k1-5} & 1.k2.5 & 1.k2.5 & 1.k2.5 & 1.k2.5 \\\hline
Finnish & {\it 3.k1.2.4.5} & 1.k2.4 & 1.k2.4 & 1.k2.4 & 1.k2.4 & 1.k2.4 \\\hline
Indonesian & 3.k1-5 & 3.k1-5 & 3.k1-5 & 1.k1.3 & 1.k1.3 & 1.k1.3 \\\hline
Lithuanian & {\it 3.k1.3-5} & {\it 3.k2.3.4} & 1.k2.3 & 1.k2.3 & 1.k2.3 & 1.k2.3 \\\hline
Bokm\aa l & 3.k2.3.5 & 3.k2.3.5 & 3.k2.3.5 & 1.k2.5 & 1.k2.5 & 1.k2.5 \\\hline
Nynorsk & {\it 3.k1-5} & {\it 3.k1.2.3.5} & 1.k2.3 & 1.k2.3 & 1.k2.3 & 1.k2.3 \\\hline
Polish & 3.k2-5 & 3.k2-5 & 3.k2-5 & 3.k2-5 & 1.k1.4 & 1.k1.4 \\\hline
Portuguese & 3.k2.3.5 & 3.k2.3.5 & 3.k2.3.5 & 2.k2.5 & 1.k2 & 1.k2 \\\hline
Romanian & {\it 3.k3-5} & 3.k1.3-5 & 3.k1.3-5 & 1.k5 & 1.k5 & 1.k5 \\\hline
Serbian.sh & {\it 3.k1.2.4.5} & 3.k2.4.5 & 3.k2.4.5 & 1.k1.3 & 1.k1.3 & 1.k1.3 \\\hline
Serbian.sr & {\it 3.k1.2.4.5} & {\it 3.k1.2.4.5} & 1.k2.5 & 1.k2.5 & 1.k2.5 & 1.k2.5 \\\hline
Slovak & {\it 3.k2.4.5} & 3.k2-5 & 3.k2-5 & 1.k2.5 & 1.k2.5 & 1.k2.5 \\\hline
Spanish & {\it 3.k1-5} & {\it 3.k1.3-5} & 1.k2.3 & 1.k2.3 & 1.k2.3 & 1.k2.3 \\\hline
Swedish & {\it 1.k1.4} & 1.k2.4 & 1.k2.4 & 1.k2.4 & 1.k2.4 & 1.k2.4 
\end{tabular}
\caption{Optimal models with tolerance for inner noise. Where the fit is
  above inherent noise the results are in {\it italics}}
\label{optimal_w_tol}
}\end{table}

Finally, let us consider the MDL results given in
Table~\ref{optimal_mdl}. These are often (9 out of 29 subcorpora) consistent
with the results obtained using Eq.~\ref{tolerable}, but never with those
obtained without considering inherent noise to be a factor. Remarkably, we
never needed more than 6 bits quantization, consistent with the general
principles of Google's TPUs \citep{Jouppi:2017} and is in fact suggestive of
an even sparser quantization regime than the eight bits employed there.  

For a baseline, we discretized the naive (nonparametric) model in the same
way.  Not only does the quantization equire on the average two bits more, but
we also have to countenance a considerably larger number of parameters to
specify the distribution within inherent noise, so that the random walk model
offers a size savings of at least 95.3\% (BNC-A) to 99.7\% (Polish).

With the random walk model, the total number of bits required for
characterizing the most complex distributions (66 for BNC-A and 60 for
Spanish) appears to be more related to the high consistency (low internal
noise) of these corpora than to the complexity of the length distributions.

\begin{table}[H]{\centering
\setlength\tabcolsep{1pt}
\hspace*{-3mm}\begin{tabular}{l*{6}{|l}}
 \multirow{2}{*}{dataset} & \multicolumn{6}{c}{best parameters for various $n$ values} \\
 & \multicolumn{1}{c|}{1k} & \multicolumn{1}{c|}{10k} & \multicolumn{1}{c|}{100k} & \multicolumn{1}{c|}{1M} & \multicolumn{1}{c|}{1G} & \multicolumn{1}{c}{$\infty$} \\\toprule
BNC-A & {\it 3.k1-5} & {\it 2.k4.5} & 2.k4.5 & 2.k4.5 & 2.k4.5 & 2.k4.5 \\\hline
BNC-B & {\it 3.k2-5} & 3.k2-5 & 1.k2.4.5 & 1.k2.4.5 & 1.k2.4.5 & 1.k2.4.5 \\\hline
BNC-C & 3.k1-5 & 1.k2.5 & 1.k2.5 & 1.k1.2.5 & 1.k2.3.5 & 1.k2.3.5 \\\hline
BNC-D & 3.k1.3-5 & 3.k2-5 & 1.k4.5 & 1.k4.5 & 1.k4.5 & 1.k4.5 \\\hline
BNC-E & {\it 3.k1-5} & 1.k2.4.5 & 1.k2.4.5 & 1.k2.4.5 & 1.k2.4.5 & 1.k2.4.5 \\\hline
BNC-F & 3.k1-5 & 1.k4.5 & 1.k4.5 & 1.k4.5 & 1.k4.5 & 1.k4.5 \\\hline
BNC-G & 3.k2-5 & 1.k2.3 & 1.k2.4.5 & 1.k2.4.5 & 1.k2.4.5 & 1.k2.4.5 \\\hline
BNC-H & {\it 3.k1.3-5} & 1.k3.5 & 1.k2.4.5 & 1.k2.4.5 & 3.k2.3.5 & 3.k2.3.5 \\\hline
BNC-J & 3.k2-5 & 3.k2-5 & 1.k1-5 & 1.k1-5 & 1.k1-5 & 1.k1-5 \\\hline
BNC-K & 3.k1-5 & 1.k2.4 & 1.k2.4.5 & 1.k2.4.5 & 1.k2.4.5 & 1.k2.4.5 \\\midrule
UMBC & {\it 3.k2-5} & 3.k2-5 & 1.k3-5 & 1.k3-5 & 1.k3-5 & 1.k3-5 \\\midrule
Catalan & {\it 3.k1.3-5} & {\it 1.k2.3} & 1.k2.3 & 1.k2.3 & 1.k2.3 & 1.k2.3 \\\hline
Croatian & {\it 3.k2.4.5} & 3.k3-5 & 1.k2.3 & 1.k1.3-5 & 1.k1.3-5 & 1.k1.3-5 \\\hline
Czech & {\it 3.k1-5} & 3.k2.4.5 & 1.k1.3-5 & 1.k1-5 & 1.k1-5 & 1.k1-5 \\\hline
Danish & {\it 3.k1-5} & 3.k1-5 & 1.k3.4 & 1.k3.4 & 1.k3.4 & 1.k3.4 \\\hline
Dutch & {\it 3.k1.2.4.5} & {\it 1.k2.4} & 1.k3-5 & 1.k3-5 & 1.k3-5 & 1.k3-5 \\\hline
Finnish & {\it 3.k1-5} & 1.k3.5 & 1.k3-5 & 1.k3-5 & 1.k3-5 & 1.k3-5 \\\hline
Indonesian & 3.k1.3-5 & 3.k2.3.4 & 1.k2.3.5 & 1.k2.3.5 & 1.k2-5 & 1.k2-5 \\\hline
Lithuanian & {\it 3.k2.3.5} & {\it 3.k2.3.4} & 1.k2.3.4 & 1.k2.3.4 & 1.k2.3.4 & 1.k2.3.4 \\\hline
Bokm\aa l & 3.k1-5 & 1.k2.4 & 1.k1.3.4 & 1.k1-5 & 1.k1-5 & 1.k1-5 \\\hline
Nynorsk & {\it 3.k2-5} & 3.k2-5 & 1.k1.4.5 & 1.k2-5 & 1.k2-5 & 1.k2-5 \\\hline
Polish & 3.k2-5 & 3.k2-5 & 1.k2.4.5 & 1.k2.4.5 & 1.k2.4.5 & 1.k2.4.5 \\\hline
Portuguese & 3.k1-5 & 3.k1-5 & 1.k2.3.4 & 1.k2-5 & 1.k2-5 & 1.k2-5 \\\hline
Romanian & 3.k1.2.4.5 & 3.k2.4.5 & 1.k3.4 & 1.k1-4 & 1.k1-4 & 1.k1-4 \\\hline
Serbian.sh & {\it 3.k1.2.4.5} & 1.k4.5 & 1.k4.5 & 2.k3-5 & 2.k3-5 & 2.k3-5 \\\hline
Serbian.sr & {\it 3.k2.4.5} & 3.k2-5 & 1.k2.3.4 & 1.k1.3-5 & 1.k1.3-5 & 1.k1.3-5 \\\hline
Slovak & {\it 3.k1-5} & 3.k1.3-5 & 1.k3-5 & 1.k3-5 & 1.k3-5 & 1.k3-5 \\\hline
Spanish & {\it 1.k2.3} & {\it 1.k2.3} & 1.k2.3.4 & 1.k2.3.4 & 1.k1-5 & 1.k1-5 \\\hline
Swedish & 3.k5 & 3.k5 & 3.k5 & 3.k5 & 3.k5 & 1.k3
\end{tabular}
\caption{Optimal models without tolerance. Fit above inherent noise in italics.}
\label{optimal_wo_tol}
}\end{table}

\section{Conclusion}
\label{sec:five}

At the outset of the paper we criticized the standard mixture Poisson length
model of Eq.~\ref{mix-poisson} for lack of a clear genesis -- there is no obvious
candidate for `arrivals' or for the mixture. In contrast, our random walk
model is based on the suggestive idea of total valency `number of things you
want to say', and we see some rather clear methods for probing this further. 

First, we have extensive lexical data on the valency of individual words, and
know in advance that e.g. color adjectives will be dependent on nouns, while
relational nouns such as {\it sister} can bring further nouns or
NPs. Combining the lexical knowledge with word frequency statistics is
somewhat complicated by the fact that a single word form may have different
senses with different valency frames, but these cause no problems for a
statistical model that convolves the two distributions.

\begin{table}[H]{\centering
\hspace*{3mm}\begin{tabular}{l|c|c|l|l|r}
 dataset & mq & nq&  tb & opt & \% size\\\toprule
BNC-A & 6 & 7 & 66 & 1.k4.5   & 4.69\\\hline
BNC-B & 4 & 5 & 40 & 2.k2.5   & 4.65\\\hline
BNC-C & 3 & 5 & 36 & 1.k1.2.4   & 3.32\\\hline
BNC-D & 2 & 3 & 6 & 1.k2   & 1.29\\\hline
BNC-E & 4 & 5 & 32 & 1.k2.5   & 3.56\\\hline
BNC-F & 2 & 5 & 16 & 1.k2.5   & 1.13\\\hline
BNC-G & 3 & 5 & 24 & 1.k1.2   & 2.45\\\hline
BNC-H & 2 & 4 & 16 & 1.k2.5   & 1.36\\\hline
BNC-J & 2 & 4 & 6 & 1.k2   & 0.51\\\hline
BNC-K & 2 & 3 & 6 & 1.k2   & 0.63\\\midrule
UMBC & 4 & 7 & 44 & 1.k4.5   & 0.88\\\midrule
Catalan & 5 & 7 & 40 & 1.k2.5   & 0.57\\\hline
Croatian & 4 & 6 & 32 & 1.k2.4   & 0.53\\\hline
Czech & 3 & 6 & 24 & 1.k2.5   & 0.41\\\hline
Danish & 3 & 6 & 24 & 1.k2.4   & 0.41\\\hline
Dutch & 5 & 7 & 40 & 1.k2.5   & 0.57\\\hline
Finnish & 4 & 7 & 48 & 1.k1.2.3   & 0.69\\\hline
Indonesian & 4 & 5 & 32 & 1.k2.4   & 0.66\\\hline
Lithuanian & 4 & 7 & 32 & 1.k2.3   & 0.46\\\hline
Bokm\aa l & 3 & 7 & 30 & 2.k2.5   & 0.43\\\hline
Nynorsk & 4 & 6 & 32 & 1.k2.3   & 1.14\\\hline
Polish & 2 & 5 &  16 & 1.k2.5   & 0.32\\\hline
Portuguese & 3 & 5 & 18 & 1.k4   & 0.36\\\hline
Romanian & 3 & 5 & 24 & 1.k1.2   & 0.48\\\hline
Serbian.sh & 4 & 6 & 32 & 1.k2.4   & 0.53\\\hline
Serbian.sr & 4 & 5 & 32 & 1.k2.5   & 0.64\\\hline
Slovak & 5 & 6 & 40 & 1.k2.3   & 0.67\\\hline
Spanish & 6 & 7 & 60 & 1.k3.4   & 0.86\\\hline
Swedish & 5 & 7 & 40 & 1.k2.3   & 0.57
\end{tabular}
\caption{Optimal models with MDL comparison (with tolerance). mq: Model
  quantization bits. nq: naive/nonparametric quantization bits. tb: total
  bits. opt: optimal model configuration. \%size: size of random walk model as
  percentage of size of nonparametric model.}
\label{optimal_mdl}
}\end{table}

Second, thanks to Universal
Dependencies\footnote{\url{http://universaldependencies.org}} we now have
access to high quality dependency treebanks where the number of dependencies
running between words $w_1, \ldots, w_k$ and $w_{k+1} \ldots w_n$, the $y$
coordinate of our random walk at $k$, can be explicitly tracked. Using these
treebanks, we could perform a far more detailed analysis of phrase or clause
formation than we attempted here.

Third, we can extend the analysis in a typologically sound manner to
morphologically more complex languages.  Using a morphologically analyzed
Hungarian corpus \citep{Oravecz:2014} we measured the per-word morpheme
distribution and per-sentence word distribution.  We found that the random sum
of `number of words in a sentence' independent copies of `number of morphemes
in a word' estimates the per-sentence morpheme distribution within inherent
noise.

Another avenue of research alluded to above would be the study of subject- and
object-control verbs and infinitival constructions, where single nouns or NPs
can fill more than one open dependency. This would complicate the calculations
in \Eqref{condition_on_first_step} in a non-trivial way. We plan to extend our
model in a future work.

One of the authors \citep{Kornai:1992} already suggested that the number of
dependencies open at any given point in the sentence must be subject to
limitations of short-term memory \citep{Miller:1956} -- this may act as a
reflective barrier that keeps asymptotic sentence length smaller than the pure
random walk model would suggest. In particular, Bernoulli and other well-known
models predict exponential decay at the high end, whereas our data shows
polinomial decay proportional to $n^{-C}$, with $C$ somewhere around 4 (in the
$3-5$ range). This is one area where our corpora are too small to draw
reliable conclusions, but overall we should emphasize that corpora already
collected (and in the case of UD treebanks, already analyzed) offer a rich
empirical field for studying sentence length phenomena, and the model
presented here makes it possible to use statistics to shed light on the
underlying grammatico-semantic structure.

\section*{Acknowledgments}

Research partially supported by National Research, Development and Innovation
Office NKFIH grant \#120145 and by National Excellence Programme
2018-1.2.1-NKP-00008: Exploring the Mathematical Foundations of Artificial
Intelligence and
NKFIH grant \#115288: Algebra and algorithms.
A hardware grant from NVIDIA Corporation is gratefully
acknowledged. GNU parallel was used to run experiments \citep{Tange2011a}.

\bibliography{slen}
\bibliographystyle{apalike}

\appendix
\renewcommand{\thesection}{\Alph{section}}

\section{Appendix}

\begin{thm}
Let us define $f$ as $x = \frac{f(x)}{F(f(x))}$ with $F(0) > 0$, then
\begin{equation}
\left[x^i\right]\left(f(x)\right)^k =
    \frac{k}{i}[x^{i-k}]F^i(x)
\end{equation}
\end{thm}
\begin{proof}
By Lagrange--Bürmann formula with composition function $H(x) = x^k$.
\end{proof}

\begin{thm}
\label{main_theorem}
In the Bayesian evidence if both the model and parameter a priori is uniform, then
\begin{gather*}
\PP(\HH_i\mid D) = \frac{\PP(D\mid\HH_i)\cdot \PP(\HH_i)}{\PP(D)} \propto 
    f(\mathbf{w}^\ast_i) + \\
    \frac{1}{n}\cdot\ln\Vol(\HH_i) +
    \frac{1}{2n} \ln\det f''(\mathbf{w}^\ast_i) + 
    \frac{d}{2n} \cdot \ln\frac{n}{2\pi}
\end{gather*}
where $f(\mathbf{w}_i)$ is the cross entropy of the measured and the modeled distributions.
See \Eqref{non_fisher_overall}.

If the augmented model (\ref{augmented_model}) is used, then \Eqref{augmented_evidence} follows.
\end{thm}
\begin{proof}
\begin{align*}
\PP(D\mid \HH_i) & \stackrel{\text{uniform a priori}}{=} \\
   & \int \PP(D\mid \mathbf{w}_i, \HH_i)\cdot \frac{1}{\Vol(\HH_i)}\dd \mathbf{w}_i = \\
   & \frac{1}{\Vol(\HH_i)}\cdot \int \prod_{x\in X} \QQ_{\mathbf{w}_i}(x)^{n_x}\dd \mathbf{w}_i =
\end{align*}
\[
\frac{\displaystyle\int \exp\Big\{-n\cdot \overbrace{\left(-\sum_{x\in X} \frac{n_x}n \cdot \ln \QQ_{\mathbf{w}_i}(x)\right)}^{f(\mathbf{w}_i)}\Big\}\dd \mathbf{w}_i}{\Vol(\HH_i)}
\]
Using Laplace method:
\begin{align*}
   \approx
   \frac{1}{\Vol(\HH_i)}\cdot e^{-n\cdot f(\mathbf{w}^\ast_i)} \cdot
                                \frac{\left(\frac{2\pi}{n}\right)^{\frac{d}{2}}}
                                {\sqrt{\det f''(\mathbf{w}^\ast_i)}}
\end{align*}
Taking $-\frac1n \ln (\bullet)$ for scaling (does not effect the relative order of the models):
\begin{align*}
   \frac1n\ln\Vol(\HH_i) + f(\mathbf{w}^\ast_i) + \frac{1}{2n} \ln\det f''(\mathbf{w}^\ast_i) + \\
   \frac{d}{2n}\cdot \ln\left(\frac{n}{2\pi}\right)
\end{align*}

As for the augmented model, the model parameters are the concatenation of the original parameters and the auxiliary parameters.
Thus the overall Hessian is the block-diagonal matrix of the original and the auxiliary Hessian.
Similarly, the overall model volume is the product of the original and the auxiliary volume.
Trivially, the logarithm of product is the sum of the logarithms.

Since the auxiliary model can fit  the uncovered part perfectly: $p_x = (1-\lambda)\cdot q_x$ on $x\notin \supp\HH_i$.
See (\ref{augmented_model}) for that $\lambda$ is the covered probability of the sample.
\begin{gather}
\nonumber
\PP(D\mid \HH'_i) = 
 -\sum_{x\in X\setminus\supp(\HH_i)} p_x\cdot \ln p_x \\
 \nonumber
-\sum_{x\in X\cap\supp(\HH_i)} p_x\cdot \ln \left(\lambda\cdot \QQ_{\mathbf{w}^\ast_i}(x)\right) + \\
\nonumber \frac{1}{n} \cdot \left( \ln\Vol(\HH_i)+\ln\Vol(\text{aux. model})\right) + \\
\nonumber \frac{1}{2 n} \cdot \ln\det\left(\text{model Hessian}\right) + \\
\nonumber  \frac{1}{2 n} \cdot  \ln\det\left(\text{aux. model Hessian}\right) + \\
\label{augmented_evidence_proof}
 \frac{d'}{2n} \cdot \ln\frac{n}{2\pi}
\end{gather}
where $d'$ is the overall parameter number.

Further, if one subtracts the entropy of the sample then
only the first two term is changed compared to \Eqref{augmented_evidence_proof}
and Eq.~\ref{augmented_evidence} follows.
\begin{gather*}
\nonumber \sum_{x\in X} p_x\cdot \ln p_x -\sum_{x\in X\setminus\supp(\HH_i)} p_x\cdot \ln p_x \\
-\sum_{x\in X\cap\supp(\HH_i)} p_x\cdot \ln \left(\lambda\cdot \QQ_{\mathbf{w}^\ast_i}(x)\right)  = \\
\sum_{x\in X\cap\supp(\HH_i)} p_x\cdot \ln \frac{p_x}{\lambda\cdot \QQ_{\mathbf{w}^\ast_i}(x)} = \\
\sum_{x\in X\cap\supp(\HH_i)} p_x\cdot \left(\ln \frac{p_x}{\QQ_{\mathbf{w}^\ast_i}(x)} + \ln\frac{1}{\lambda}\right) = \\
\lambda\cdot (-\ln\lambda) +
\sum_{x\in X\cap\supp(\HH_i)}p_x\cdot \ln \frac{p_x}{\QQ_{\mathbf{w}^\ast_i}(x)}
\end{gather*}
q.v. Definition~\ref{generalized_kl_def}.
\end{proof}

\end{document}